\documentclass{article}

     \PassOptionsToPackage{numbers}{natbib}

     \usepackage[preprint]{neurips_2019}

\usepackage[utf8]{inputenc} 
\usepackage[T1]{fontenc}    
\usepackage{hyperref}       
\usepackage{url}            
\usepackage{booktabs}       
\usepackage{amsfonts}       
\usepackage{nicefrac}       
\usepackage{microtype}      

\usepackage{amsmath, amssymb, amsthm, enumerate, bm, capt-of}
\usepackage{hyperref}
\usepackage[capitalise]{cleveref}

\usepackage{relsize}
\usepackage{xcolor}
\usepackage{fp, tikz, pgfplots}
\usetikzlibrary{arrows}
\usetikzlibrary{shapes}
\usetikzlibrary{patterns}
\usetikzlibrary{decorations.pathmorphing}
\usetikzlibrary{decorations.pathreplacing}
\usetikzlibrary{decorations.shapes}
\usetikzlibrary{decorations.text}
\usetikzlibrary{shapes.misc}
\usetikzlibrary{backgrounds}
\usepgfplotslibrary{external}
\usepgfplotslibrary{fillbetween}

\usepackage{etoolbox} 

\newtheorem{lemma}{Lemma}[section]
\newtheorem{theorem}{Theorem}[section]
\newtheorem{corollary}{Corollary}[section]

\newtheorem{remark}{Remark}[section]

\makeatletter
\newcommand{\subalign}[1]{%
	\vcenter{%
		\Let@ \restore@math@cr \default@tag
		\baselineskip\fontdimen10 \scriptfont\tw@
		\advance\baselineskip\fontdimen12 \scriptfont\tw@
		\lineskip\thr@@\fontdimen8 \scriptfont\thr@@
		\lineskiplimit\lineskip
		\ialign{\hfil$\m@th\scriptstyle##$&$\m@th\scriptstyle{}##$\crcr
			#1\crcr
		}%
	}
}
\makeatother

\pgfplotsset{width=5\columnwidth /5, compat = 1.13, 
	height = 47\columnwidth /100, grid= major, 
	legend cell align = left, ticklabel style = {font=\scriptsize},
	every axis label/.append style={font=\small},
	legend style = {font=\tiny},title style={yshift=-7pt, font = \small} }

\tikzset{cross/.style={cross out, draw=black, minimum size=10*(#1-\pgflinewidth), inner sep=0pt, outer sep=0pt},cross/.default={1pt}}

\title{Posterior Variance Analysis of Gaussian Processes
with Application to Average Learning Curves}

\author{%
	Armin Lederer \\
	Technical University of Munich\\
	\texttt{armin.lederer@tum.de}
	\And
	Jonas Umlauft \\
	Technical University of Munich\\
	\texttt{jonas.umlauft@tum.de}
	\And
	Sandra Hirche\\
	Technical University of Munich\\
	\texttt{hirche@tum.de}
}

\begin{document}

\maketitle

\begin{abstract}
	The posterior variance of Gaussian processes is a 
	valuable measure of the learning error which is 
	exploited in various applications such as 
	safe reinforcement learning and control 
	design. However, suitable analysis of the 
	posterior variance which captures its behavior 
	for finite and infinite number of 
	training data is missing. This paper derives a novel bound for the
	posterior variance function which requires only local information because it 
	depends only on the number of training samples in the proximity of a considered 
	test point. 
	Furthermore, we prove sufficient conditions 
	which ensure the convergence of the posterior 
	variance to zero. Finally, we demonstrate that the extension of our 
	bound to an average learning bound outperforms existing 
	approaches.\looseness=-1
\end{abstract}

\section{Introduction}
Gaussian process (GP) regression is a probabilistic supervised 
machine learning method that bases on Bayesian 
principles~\citep{Rasmussen2006}. GP 
regression generalizes efficiently with 
little training data, which makes it appealing to 
real world applications with limited amount of 
training data. Therefore, it has gained
increasing attention in the field of reinforcement
learning and system identification for control
design in recent years. Especially, when safety
guarantees are necessary, GPs 
are the method of choice in active 
and reinforcement learning~\cite{Berkenkamp,
Berkenkamp2016,Berkenkamp2017, Koller2018}
as well as control~\cite{Berkenkamp2015, Umlauft2017,
Beckers2018, Umlauft2018a, Umlauft2018, Helwa2018}.
These safety critical applications have in common that
they rely on the posterior variance for deriving
uniform error bounds~\cite{Srinivas2012, Chowdhury2017a}.
However, the behavior of the posterior variance when data points are added 
on-line, e.g. during control tasks, has barely been analyzed formally due to a
lack of suitable bounds. Therefore, there is generally
little understanding of the interaction between 
learning and control in feedback systems, which is 
crucial to provide guarantees for the control error. \looseness=-1

Considering uniform training data distributions, 
the average posterior variance of GPs
has extensively been studied, see~\cite{Sollich1999,
 Malzahn2001, Sollich2002, LeGratiet2014}. The 
mapping between this average variance and the number
of training samples is usually referred to as \textit{average
learning curve} and it is used to evaluate the 
generalization properties of GPs. 
Although average learning curves have been applied to few applications, 
e.g.,~\cite{Xu2011}, they provide important theoretical insights to the 
learning behavior of GPs~\cite{Schulz2015,Ueno2018}. This understanding
can be exploited in sparse GP approximations 
in a similar way as proposed for PAC-Bayesian error bounds in~\cite{Reeb2018}.
Furthermore, active learning and experiment design
can be an application scenario of average learning curves since 
common criteria such as the mutual information~\cite{Krause2008} also measure 
the generalization error. However, the framework developed 
for average learning curves is directly applicable to continuous input 
spaces, while it is difficult to evaluate the 
mutual information in this setting.\looseness=-1

The contribution of this paper is a novel bound for 
the posterior variance of GPs with
Lipschitz continuous covariance kernels. 
and demonstrate and improvement of the bound for a more specific class of 
kernels. Furthermore, we derive sufficient 
conditions for the generation of training data which ensure the 
convergence of our posterior variance bounds to zero and 
investigate criteria for probability distributions such that
the convergence conditions are satisfied. Finally, we
show a straight forward extension of our 
bounds to average learning curve bounds and 
compare our results to numerically obtained 
approximations. In fact, our average learning 
curve bound can be seen as generalization 
of the approach in~\cite{Williams2000}, which our method 
outperforms.

The remaining paper is structured as follows: In
\cref{sec:GP}, we provide an overview of related work on
posterior variance bounds and average learning curves. 
Novel posterior variance 
bounds and necessary conditions on their convergence 
are derived in \cref{sec:postVar}. Finally, the 
derived bounds are compared to 
approximations in \cref{sec:numEval}.\looseness=-1

\section{Related Work}
\label{sec:GP}

\subsection{Gaussian Process Regression}
A Gaussian process is a stochastic process such 
that any finite number of outputs\footnote{Vectors/matrices
	are denoted by lower/upper case bold symbols, the 
	$n\times n$ identity matrix by~$\bm{I}_n$, the Euclidean norm 
	by~$\|\cdot\|$, sets by upper case 
	black board bold letters. Sets restricted to positive 
	numbers have an indexed~$+$, e.g.~$\mathbb{R}_+$ for 
	all positive real valued numbers. The cardinality of 
	sets is denoted by~$|\cdot|$. The expectation operator
	$E[\cdot]$ can have an additional index to specify 
	the considered random variable. 
	Class~$\mathcal{O}$ notation is used to provide 
	asymptotic upper bounds on functions. The ceil and floor
	operator are denoted by $\lceil\cdot\rceil$ and $\lfloor\cdot\rfloor$,
	respectively.\looseness=-1} 
\mbox{$\{y_1,\ldots,y_M\}\subset\mathbb{R}$}
is assigned a joint Gaussian distribution with
prior mean~$0$ and covariance defined through the kernel 
$k:\mathbb{R}^d\times\mathbb{R}^d\rightarrow\mathbb{R}$
\cite{Rasmussen2006}. Therefore, the training outputs 
$y^{(i)}$ can be considered as observations of a sample 
function~$f:\mathbb{X}\subset\mathbb{R}^d\rightarrow\mathbb{R}$ of the GP 
distribution perturbed by i.i.d. zero mean Gaussian noise
with variance~$\sigma_n^2$. Regression is performed by conditioning the 
prior GP distribution on the training data 
$\mathbb{D}_N=\{(\bm{x}^{(i)},y^{(i)})\}_{i=1}^N$ and a test point~$\bm{x}$. 
The conditional posterior 
distribution is again Gaussian and can be calculated analytically.
For this reason, we define the kernel matrix~$\bm{K}_N$ and the the 
kernel vector~$\bm{k}_N(\bm{x})$ through~$K_{N,ij}=k(\bm{x}^{(i)},\bm{x}^{(j)})$ 
and~$k_{N,i}(\bm{x})=k(\bm{x},\bm{x}^{(i)})$, respectively, with~$i,j=1,\ldots,N$. 
Then, the posterior
mean~$\mu_N(\cdot)$ and variance~$\sigma_N^2(\cdot)$ are given by
\begin{align}
\mu_N(\bm{x})&=\bm{k}_N^T(\bm{x})\bm{A}_N^{-1}\bm{y}_N,\\
\sigma_N^2(\bm{x})&=k(\bm{x},\bm{x})-
\bm{k}_N^T(\bm{x})\bm{A}_N^{-1}\bm{k}_N(\bm{x}),
\label{eq:var}
\end{align}
where~$\bm{A}_{N}=\bm{K}_N+\sigma_n^2\bm{I}_{N}$
denotes the data covariance matrix and~$\bm{y}_N = [y^{(1)} \cdots y^{(N)}]^T$.

\subsection{Posterior Variance Bounds and Average Learning Curve Bounds}
A common measure to analyze the learning speed of 
GPs are average learning curves, which 
are also called integrated mean squared 
errors~\cite{LeGratiet2014}. Under the assumption that~$y^{(i)}$ are noisy 
observations of a function~$f(\cdot)$, which is a sample 
function from the GP, the 
mean squared error of the posterior GP 
is given by~$E_y[(y-\mu_N(\bm{x}))^2]=
\sigma_N^2(\bm{x})+\sigma_n^2$.
The average learning curve is obtained from this 
equation by taking the expectation with respect to 
the test point~$\bm{x}$ and the input training data 
$\mathbb{D}_N^x=\{\bm{x}^{(i)} \}_{i=1}^N$, i.e.,~$e(N)=E_{\mathbb{D}_N^x}\left[E_{\bm{x}}
\left[\sigma_N^2(\bm{x})+\sigma_n^2\right]\right]$.
For notational simplicity of the following derivations, we consider the uniform
distributions over the unit interval $\mathbb{X}=[0,1]$ in the following. However, 
all derivations can be extended to higher dimensional state spaces and
other distributions even though it is a little technical.

A simple approach to obtain a learning curve bound 
for GPs with isotropic kernels, 
which only depend on the distance between
their arguments~$k(x,x')=k(\|x-x'\|)$,
 proposed 
in~\cite{Williams2000} bases on the idea to 
consider only the training samples~$x^{(i)}$ 
closest to~$x$ in the variance calculation. This 
approach leads to a valid posterior variance bound 
since the posterior variance cannot increase by adding
training samples~\cite{Vivarelli1998}. Considering only
the nearest training sample in the calculation of the 
posterior variance \eqref{eq:var} directly leads to\looseness=-1
\begin{align}
\sigma_N^2(x)\leq \sigma_1^2(x)=
k(0)-\frac{k^2(\tau)}{k(0)
	+\sigma_n^2}
\label{eq:sig1bound},
\end{align}
with~$\tau$ being the minimal Euclidean distance between 
$x$ and the training data set 
$\mathbb{D}_N^x$, 
i.e.~$\tau\!=\!\min_{x'\in\mathbb{D}_N^x}\|x\!-\!x'\|$. 
Assume that the training data is ordered by increasing value of~$x$ and divide 
the 
unit interval in~$N$ segments such that the boundaries
are given by~$a_1\!=\!0$,~$a_i\!=\!(x^{(i)}\!+\!x^{(i-1)})/2$,~$b_{i-1}\!=\!(x^{(i)}\!+\!x^{(i-1)})/2$,~$b_N\!=\!1$ for
$i\!=\!2,\ldots,N$. Then, the expectation with respect to the test points can be 
approximated by~$E_{x}[\sigma_N^2(x)]\leq 
\sum_{i=1}^{N}\int_{a_i}^{b_{i}}
\sigma_{1}^2(\xi)\mathrm{d}\xi$.
Exploiting \eqref{eq:sig1bound} and symmetry of the covariance, it is 
straightforward to show that these integrals only depend on the distance~$\delta$ 
between 
training samples. Therefore, the expectation with respect to the training 
data~$\mathbb{D}_N^x$ reduces 
to an expectation with respect to~$\delta$, such that the average learning curve 
can be 
bounded by\looseness=-1
\begin{align}
e(N)&\leq\bar{e}_1(N)=  k(0)+\sigma_n^2-2
\frac{E_{\delta}\left[\int_{0}^{\delta} k^2(\tau)
	\mathrm{d}\tau\right]}{k(0)+\sigma_n^2}
-2(N-1)\frac{E_{\delta}\left[\int_{0}^{\frac{\delta}{2}} k^2(\tau)
	\mathrm{d}\tau\right]}
{k(0)+\sigma_n^2}.
\label{eq:alc_1}
\end{align}
The expectations in this bound can be calculated analytically for some 
kernels since the difference~$\delta$ between adjacent points follows first 
order statistics, hence, we have~$p(\delta)=N(1-\delta)^{N-1}$. However,
they are typically computed numerically~\cite{Williams2000}.

When considering the two closest training samples, the
inverse in \eqref{eq:var} still leads to a simple expression 
which leads to the following posterior variance bound
\begin{align}
	\sigma_N^2(x)\leq \sigma_2^2(x)=
	k(0)-\frac{k(0)+\sigma_n^2(k^2(\tau_2)+k^2(\tau_1))-2k(\eta)k(\tau_1)k(\tau_2)}
	{(k(0)+\sigma_n^2)^2-k^2(\eta)},
	\label{eq:sig2bound}
\end{align}
where~$\tau_1$ and~$\tau_2$ are the distances to the two closest training 
samples and 
$\delta$ is the distance between the two closest training samples. By defining segments 
with~$a_1\!=\!0$,~$a_i\!=\!x^{(i-1)}$,~$b_{i-1}\!=\!x^{(i-1)}$,~$b_{N+1}\!=\!1$ for~$i\!=\!2,\ldots,N+1$, 
\eqref{eq:sig2bound}
and symmetry of the kernel can be exploited to derive an expression for the expectation
with respect to the test points which depends only on the distance~$\delta$ 
between
training points, such that we obtain the average learning curve bound
\begin{align}
	\bar{e}_2(N)\!=\!k(0)\!+\!\sigma_n^2
	\!-\!2(N-1)E_{\delta}\!\!\left[\frac{\!\int_0^{\delta}(k(0)\!+\!\sigma_n^2)k^2(\tau)\!+\!k(\delta)k(\tau)k(\delta\text{-}\tau)\mathrm{d}\tau}
	{(k(0)\!+\!\sigma_n^2)^2-k^2(\delta)}\!\right] \!
	\!-\!2\frac{ E_{\delta}\!\left[\int_{0}^{\delta} k^2(\tau)
		\mathrm{d}\tau\right]}{k(0)+\sigma_n^2}.
	\label{eq:alc2}
\end{align}
Although both bounds are relatively tight for small
numbers of training data, they do not converge 
to the asymptotic value of the average learning curve
$\sigma_n^2$. Instead, the 
bound~$\bar{e}_1(N)$ has been shown to converge to 
$\sigma_n^2(2+\sigma_n^2)/(1+\sigma_n^2)$, while
$\bar{e}_2(N)$ converges to~$\sigma_n^2(3+\sigma_n^2)/(2+\sigma_n^2)$
\cite{Williams2000}. Therefore, these bounds do not provide any insight when 
analyzing the learning behavior with large data sets.

\subsection{Literature Review}

Some posterior variance bounds for GP 
regression have been developed as intermediate results in the 
context of Bayesian optimization, e.g.,~\cite{Shekhar2018}.
However, in this area, isotropic kernels are typically used which hinders the 
application outside of this field. For noise-free interpolation,
the posterior variance has been analyzed using spectral 
methods~\cite{Stein1999}. While the asymptotic behavior
can be analyzed efficiently
with such methods, they are not suited to bound the 
posterior variance for specific training data sets. In the 
context of noise-free interpolation, many bounds from the 
area of scattered data approximation can be applied due to 
the equivalence of the posterior variance and the power
function~\cite{Kanagawa2018}. Therefore, classical results
\cite{Wu1993, Wendland2005, Schaback2006}
as well as newer findings~\cite{Beatson2010, Scheuerer2013}
can be directly used for GP interpolation. 
However, it is typically not clear how these results can
be generalized to regression with noisy observations.\looseness=-1

For the derivation of average learning curves, many 
different approaches have been pursued in literature. 
A common method to approximate learning curves builds
on spectral methods, e.g.,~\cite{Sollich1999, Malzahn2001,
 Sollich2002, Sarkka2013, LeGratiet2014}. This approach 
has also been extended to special situations such as 
learning on graphs~\cite{Urry2013} 
and multi-task learning~\cite{Chai2009, Ashton2012}.
However, these approaches cannot be employed in any formal 
proof on the generalization properties of GPs since they 
only describe the approximate learning behavior. Therefore, some work 
has focused on deriving strict upper and lower bounds 
for average learning curves \cite{Opper1999, Williams2000}.
However, the upper bounds in
\cite{Williams2000} suffer from the disadvantage, that
they can only capture the learning behavior for few training
samples. Hence, upper bounds for average learning curves
are missing that are capable of describing the learning 
behavior for small as well as large data sets.
\looseness=-1
	
\section{Posterior Variance of Gaussian Processes}
\label{sec:postVar}

Despite a wide variety of literature 
on average learning curves and posterior variance 
bounds for isotropic kernels, learning curve bounds
and general posterior variance bounds have gained
far less attention. Exploiting ideas from existing
posterior variance bounds, we derive in
\cref{subsec:varBound} an upper bound on the posterior 
variance, which depends on the number of samples in 
the neighborhood of the test point~$\bm{x}$. In \cref{subsec:probDist} we 
derive sufficient 
conditions on probability distributions of the training 
data that ensure the convergence of our bound. 
Finally, we demonstrate how the derived 
bound for isotropic kernels can be applied 
to average learning curve bounds of 
GP in \cref{subsec:app_alc}.\looseness=-1

\subsection{Posterior Variance Bound and Asymptotic Behavior}
\label{subsec:varBound}

The central idea in deriving an upper bound for 
the posterior variance of a GP lies 
in the observation that data close to a test point 
$\bm{x}$ usually lead to the highest decrease in the posterior 
variance. Therefore, it is natural to consider only 
training data close to the test point in the bound as 
more and more data is acquired. The following theorem 
formalizes this idea. The proofs for all the following
theoretical results can be found in the supplementary material.

\begin{theorem}
	\label{th:var_bound}
	Consider a GP with Lipschitz 
	continuous kernel~$k(\cdot,\cdot)$ with Lipschitz constant~$L_k$, an input 
	training data set~$\mathbb{D}_{N}^x=\{\bm{x}^{(i)}\}_{i=1}^{N}$ 
	and observation noise variance~$\sigma_n^2$. Let 
	$\mathbb{B}_{\rho}(\bm{x})=\{ 
	\bm{x}'\in\mathbb{D}_N^x:~\|\bm{x}'-\bm{x}\|\leq 
	\rho \}$ denote the training data set restricted to a 
	ball around~$\bm{x}$ with radius~$\rho$. Then, for 
	each~$\bm{x}\in\mathbb{X}$ and 
	$\rho\leq k(\bm{x},\bm{x})/L_k$, 
	the posterior variance is bounded by\looseness=-1
	\begin{align}
	\sigma_N^2(\bm{x})\leq \frac{(4L_k\rho-L_k^2\rho^2)\left|
		\mathbb{B}_{\rho}(\bm{x})\right|k(\bm{x},\bm{x})
		+\sigma_n^2k(\bm{x},\bm{x})}{\left|
		\mathbb{B}_{\rho}(\bm{x})\right|(k(\bm{x},\bm{x})
		+2L_k\rho)+\sigma_n^2}.
	\label{eq:sigbound}
	\end{align}
\end{theorem}
The parameter~$\rho$ can be interpreted as information
radius, which defines how far away from a test point 
$\bm{x}$ training data is considered to be informative. However, this information radius is conservative
as all the data points with smaller radius are treated in
the theorem as if they had a distance of~$\rho$ to the 
test point. Therefore, a large~$\rho$ has the advantage
that many training points are considered, while a small
$\rho$ is beneficial if sufficiently many training samples
are close to the test point~$\bm{x}$.

Note, that Theorem~\ref{th:var_bound} is very general as it is merely 
restricted 
to Lipschitz continuous kernels, which is a common 
property of kernels for regression~\cite{Rasmussen2006}. 
This generality comes at the price of tightness of the 
bound and tighter bounds exist under additional assumptions
, e.g., the bound in~\cite{Shekhar2018} for isotropic, decreasing kernels,
 which have non-positive derivatives 
$\frac{\partial}{\partial \tau}k(\tau)\leq 0$,~$\tau\geq 0$. However, this 
bound can directly be derived 
from \cref{th:var_bound}, which leads to the following corollary. 
\begin{corollary}
	\label{cor:var_bound}
	Consider a GP with isotropic, decreasing 
	covariance kernel~$k(\cdot)$, an input 
	training data set~$\mathbb{D}_{N}^x=
	\{\bm{x}^{(i)}\}_{i=1}^{N}$ and observation noise 
	variance~$\sigma_n^2$. Let 
	$\mathbb{B}_{\rho}(\bm{x})=\{ \bm{x}'
	\in\mathbb{D}_N^x:~\|\bm{x}'-\bm{x}\|\leq \rho \}$ 
	denote the training data set restricted to a ball 
	around~$\bm{x}$ with radius~$\rho$. Then, 
	for each~$\bm{x}\in\mathbb{X}$, the 
	posterior variance is bounded by
	\begin{align}
	\sigma_N^2(\bm{x})\leq k(0)-\frac{k^2(\rho)}{k(0)
		+\frac{\sigma_n^2}{|\mathbb{B}_{\rho}(\bm{x})|}}.
	\label{eq:isobound}
	\end{align}
\end{corollary}

In addition, \cref{th:var_bound} can also be used 
for an asymptotic analysis of the posterior variance, 
i.e.,~$\lim_{N\rightarrow\infty}\sigma_N^2(\bm{x})$.
Even though the limit of infinitely many training
data cannot be reached in practice, this analysis
is important because it helps to determine the 
amount of training data which is necessary to 
achieve a desired posterior variance. In the following
corollary, we provide necessary conditions that
ensure the convergence to zero of the bound~\eqref{eq:sigbound}.
\begin{corollary}
	\label{th:varvan}
	Consider a GP with Lipschitz 
	continuous kernel~$k(\cdot,\cdot)$ , an infinitely large 
	input training data set~$\mathbb{D}_{\infty}^x=\{ \bm{x}^{(i)} 
	\}_{i=1}^{\infty}$ and the observation noise 
	variance~$\sigma_n^2$. Let 
	$\mathbb{D}_{N}^x=\{ \bm{x}^{(i)} \}_{i=1}^{N}$ 
	denote the subset of the first~$N$ input training 
	samples and let~$L_k$ be the Lipschitz constant 
	of kernel~$k(\cdot,\cdot)$. Furthermore, let 
	$\mathbb{B}_{\rho}(\bm{x})=\{ \bm{x}'
	\in\mathbb{D}_N^x:~\|\bm{x}'-\bm{x}\|\leq \rho \}$ 
	denote the training data set restricted to a ball 
	around~$\bm{x}$ with radius~$\rho$. If there 
	exists a function 
	\mbox{$\rho: \mathbb{N}\rightarrow \mathbb{R}_+$} such that 
	\begin{align}
	\rho(N)&\leq \frac{k(\bm{x},\bm{x})}{L_k}
	\quad \forall N\in\mathbb{N}\\
	\lim\limits_{N\rightarrow\infty}\rho(N)&=0\\
	\lim\limits_{N\rightarrow\infty}\left| 
	\mathbb{B}_{\rho(N)}(\bm{x}) \right|&=\infty
	\label{eq:cond2}
	\end{align}
	holds, the posterior variance at~$\bm{x}$ 
	converges to zero, i.e.~$\lim_{N\rightarrow\infty}\sigma_N(\bm{x})=0$.
\end{corollary}
Although it might be unintuitive that the number of 
training samples in a ball with vanishing radius has 
to reach infinity in the limit of infinite training 
data, this is not a restrictive condition. 
Deterministic sampling strategies can satisfy it, 
e.g. if a constant fraction of the samples lies on 
the considered point~$\bm{x}$ or if the maximally allowed 
distance of new samples reduces with the total number 
of samples. Furthermore, this condition is satisfied 
for a wide class of probability distributions for 
sufficiently slowly vanishing radius~$\rho(N)$ as 
shown in the following section.

\begin{remark}
\cref{th:varvan} does not require dense sampling in a neighborhood of the test 
point~$\bm{x}$. In fact, the 
conditions on the training samples in \cref{th:varvan}
are satisfied if the data is sampled densely , e.g., from 
a manifold which contains the test point~$\bm{x}$, 
such as a line through~$\bm{x}$.
\end{remark}
\subsection{Conditions on Probability Distributions for Asymptotic Convergence}
\label{subsec:probDist}
For fixed~$\rho$ it is well known that the number 
of training samples inside the ball~$\mathbb{B}_{\rho}
(\bm{x})$ converges to its expectation due to the 
strong law of large numbers. Therefore, it is 
sufficient to analyze the asymptotic behavior of the 
expected number of samples inside the ball instead of 
the actual number for fixed~$\rho$. However, it is not 
clear how fast the radius~$\rho(N)$ is allowed to 
decrease in order to ensure convergence of 
$|\mathbb{B}_{\rho(N)}(\bm{x})|$ to its expected 
value. The following theorem shows that the 
admissible order of~$\rho(N)$ depends on the local 
behavior of the density~$p(\cdot)$ around~$\bm{x}$.

\begin{theorem}
	\label{th:ball}
	Consider a sequence of points~$\mathbb{D}_{\infty}^x=
	\{ \bm{x}^{(i)} \}_{i=1}^{\infty}$ which is generated by 
	drawing from a probability distribution with 
	density~$p(\cdot)$. If there exists a 
	non-increasing function~$\rho: \mathbb{N}
	\rightarrow \mathbb{R}_+$ and constants~$c,\epsilon\in\mathbb{R}_+$ 
	such that
	\begin{align}
	\lim\limits_{N\rightarrow\infty}\rho(N)&=0\\
	\int_{\{\bm{x}'\in\mathbb{X}:\|\bm{x}
		-\bm{x}'\|\leq \rho(N)\}}p(\bm{x}')\mathrm{d}\bm{x}'
	&\geq cN^{-1+\epsilon},
	\label{eq:cond2new}
	\end{align}
	then, the sequence 
	$|\mathbb{B}_{\rho(N)}(\bm{x})|$ goes to infinity 
	almost surely, i.e.~$\lim_{N\rightarrow\infty}
	|\mathbb{B}_{\rho(N)}(\bm{x})|=\infty~ a.s.$
\end{theorem}
Similarly to \cref{th:var_bound}, \cref{th:ball}
is formulated very general to be applicable to a wide 
variety of probability distributions. However, under 
additional assumptions condition \eqref{eq:cond2new} 
can be simplified. This is exemplary shown for 
probability densities which are positive in a 
neighborhood of the considered point~$\bm{x}$.\looseness=-1

\begin{corollary}
	\label{cor:varvan}
	Consider a sequence of points~$\mathbb{D}_{\infty}^x
	=\{ \bm{x}^{(i)} \}_{i=1}^{\infty}$ which is generated by 
	drawing from a probability distribution with 
	density~$p(\cdot)$, such that~$p(\cdot)$ is 
	positive in a ball around~$\bm{x}$ with any radius 
	$\xi\in\mathbb{R}_+$, i.e.
	\begin{align}
	p(\bm{x}')>0\quad\forall \bm{x}'\in\{ \bm{x}': 
	\|\bm{x}-\bm{x}'\|\leq \xi  \}.
	\end{align}
	Then, for all non-increasing 
	functions~$\rho:\mathbb{N}\rightarrow\mathbb{R}_+$ 
	for which exist~$c,\epsilon\in\mathbb{R}_+$ such 
	that
	\begin{align}
	\rho(N)&\geq cN^{-\frac{1}{d}+\epsilon}\quad 
	\forall N\in\mathbb{N}\\
	\lim\limits_{N\rightarrow\infty}\rho(N)&=0
	\end{align}
	it holds that~$\lim_{N\rightarrow\infty}|\mathbb{B}_{\rho(N)}
	(\bm{x})|=\infty~ a.s.$
\end{corollary}
This corollary shows that it is relatively 
simple to allow the maximum decay rate of 
$\rho(N)\approx N^{-1}$ for scalar inputs. 
For higher dimensions~$d$ however, it cannot 
be achieved and the allowed decay rate decreases 
exponentially with~$d$. Yet, this is merely a 
consequence of the curse of dimensionality.\looseness=-1

\subsection{Application to Average Learning Curves}
\label{subsec:app_alc}
\setlength{\textfloatsep}{12pt}
\setlength{\floatsep}{6pt}

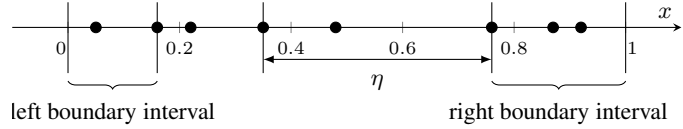
\begin{figure}
	\center
	\begin{minipage}{0.75\textwidth}
		\tikzset{fontscale/.style = {font=\footnotesize}}
		\center
		\begin{tikzpicture}
		\begin{axis}[xmin=-0.1,xmax=1.1,ymin=-6,ymax=5, samples=100,
		grid=none, axis y line=left, axis x line=middle, hide y axis,
		xlabel=$x$,ylabel={$\sigma_N^2$}, xtick={0,0.2,0.4,0.6,0.8,1}, 
		xticklabels={$0~~~$,$0.2$,$0.4$,$0.6$,$0.8$,$~~~1$}
		]
		\addplot[only marks, mark=*] coordinates {(0.05,0) (0.16,0)
			(0.22,0) (0.35,0) (0.48,0) (0.76,0) (0.87,0) (0.92,0)};
		\draw (axis cs: 0,1.1)--(axis cs: 0,-2.0);
		\draw (axis cs: 0.16,1.1)--(axis cs: 0.16,-2.0);
		\draw (axis cs: 0.35,1.1)--(axis cs: 0.35,-2.0);
		\draw (axis cs: 0.76,1.1)--(axis cs: 0.76,-2.0);
		\draw (axis cs: 1,1.1)--(axis cs: 1,-2.0);
		\draw [decorate, decoration={brace,amplitude=3pt,mirror,raise=2pt}, yshift=-1pt]
		(axis cs:0.0,-2.0) -- (axis cs:0.16,-2.0) node [black,midway,xshift=0cm,yshift=-0.5cm,align=center,fontscale=1] {left boundary interval};
		\draw [decorate, decoration={brace,amplitude=3pt,mirror,raise=2pt}, yshift=-1pt]
		(axis cs:0.76,-2.0) -- (axis cs:1,-2.0) node [black,midway,xshift=0cm,yshift=-0.5cm,align=center,fontscale=1] {right boundary interval};
		\draw[<->,>=latex] (0.35,-1.4)--(0.76,-1.4) node [black, midway,yshift=-0.3cm,align=center,fontscale=1] {$\eta$};		\end{axis}
		\end{tikzpicture}
	\end{minipage}
	\vspace{-0.2cm}
	\caption{Fixed training data set with~$N=8$ and division such that inner 
	intervals have~$n=3$ samples, 
		left boundary interval~$n_l=2$ samples and right boundary 
		interval~$n_r=3$ samples}
	\label{fig:intervals}
\end{figure}
Both posterior variance bounds in~\cite{Williams2000}
suffer from the fact that they do not converge to zero
in the limit of infinite training data. However, the 
idea used in \cite{Williams2000} to derive \eqref{eq:sig1bound} and 
\eqref{eq:sig2bound} is the same as in \cref{th:var_bound}.
In fact, \eqref{eq:sig1bound} can be seen as a 
special case of our bound in \cref{cor:var_bound} with 
\mbox{$|\mathbb{B}_{\rho}(\bm{x})|=1$}. Therefore, it is natural to
employ \eqref{eq:isobound} for the derivation of 
average learning curve bounds by choosing $\rho$ such that 
\mbox{$|\mathbb{B}_{\rho}(\bm{x})|=n>1$}. Furthermore, 
we divide the unit interval in~$m=\left\lceil 
(N-2n+1)/(n-1) \right\rceil$ inner sections and two
boundary sections as depicted in \cref{fig:intervals} for~$N=8$ 
and~$n=3$. The inner sections are chosen
such that each of them starts and ends at a 
training sample and contains exactly~$n>1$ of them. 
The remaining~$N-m(n-1)+1$ training points are 
divided fairly among the two boundaries: the left 
boundary section contains~$n_l=\left\lfloor 
(N-m(n-1)+1)/2 \right\rfloor$ and the right 
boundary section contains~$n_r=\left\lceil 
(N-m(n-1)+1)/2 \right\rceil$ training samples 
such that they stop and start at training samples, 
respectively. Hence, we can bound the average 
learning curve by\looseness=-1
\begin{align}
\bar{e}_{\rho}(N)\!&=\! k(0)\!+\!\sigma_n^2\!-\!2m \int_{0}^{1}\!\frac{I_n(\delta)}{k(0)+
	\frac{\sigma_n^2}{n}}\mathrm{d}\delta\!-\!2\int_{0}^{1}\!\frac{I_{n_l+1}(\delta)}{k(0)+
	\frac{\sigma_n^2}{n_l}}\mathrm{d}\delta\!-\!2\int_{0}^{1}\!\frac{I_{n_r+1}(\delta)}{k(0)+
	\frac{\sigma_n^2}{n_r}}\mathrm{d}\delta,
\label{eq:alc_rho}\\
\text{where } I_{n}(\delta)&\!=\!\binom{N}{n-1}(1\!-\!\delta)^{N-n-1}
\delta^{n-2}\int_{\frac{\delta}{2}}^{\delta} \!
k^2(\rho)\mathrm{d}\rho,
\label{eq:int_alc}
\end{align}
due to the fact that the distance between~$n$ 
training samples follows order statistics. Note 
that the integral in \eqref{eq:int_alc} has the lower boundary 
$\frac{\delta}{2}$ since this is the minimal distance to 
either boundary. Therefore, the 
maximal distance to a training point inside the 
considered section varies between~$\frac{\delta}{2}$ 
and~$\delta$. Due to \cref{cor:varvan}, \eqref{eq:isobound} converges to 
zero for uniformly sampled training data with a suitably defined $\rho(N)$. Hence,
\eqref{eq:alc_rho} must also converge to zero for this information
radius $\rho(N)$ and is therefore capable of describing the learning
behavior for both small and large training data sets.
  \looseness=-1

\section{Numerical Evaluation}
\label{sec:numEval}
In this section we illustrate the behavior 
of the proposed bounds. \cref{subsec:VarUniform} 
compares our variance bounds to the exact posterior
variance for uniformly sampled training data
 and training data sampled from a distribution which 
 vanishes at the considered point. In 
\cref{subsec:alc} we demonstrate the derived bounds 
on average learning curves for isotropic kernels
and compare them to existing approaches.\looseness=-1

\subsection{Posterior Variance Bounds}
\label{subsec:VarUniform}
We compare the bounds in \cref{th:var_bound} and 
\cref{cor:var_bound} to the exact 
posterior variance for GPs with a 
squared exponential, a Mat\'ern kernel with 
$\nu=\frac{1}{2}$, a polynomial kernel with $p=3$ and 
a neural network kernel. The posterior variance is 
evaluated at the point $x=1$ for a uniform training 
data distribution $\mathcal{U}([0.5,1.5])$. 
Furthermore, the 
length scale of the kernels 
is set to $l=1$ where applicable and the noise 
variance is set to~$\sigma_n^2=0.1$. 
In order to obtain a good value 
for the information radius $\rho$, consider the 
following approximation of \eqref{eq:isobound} for 
isotropic kernels 
\begin{align}
\hat{\sigma}_{\rho}^2(1)\approx k(0)-
\frac{k^2(\rho)}{k(0)}+\frac{k(0)\sigma_n^2}{N\rho k(0)
	+\sigma_n^2},
\end{align}
where we use the expectation of 
$E[|\mathbb{B}_{\rho}(1)|]=N\rho$ instead of the 
random variable~$|\mathbb{B}_{\rho}(1)|$. For the 
squared exponential kernel the Taylor 
expansion around~$\rho=0$ yields
\begin{align}
k(0)-\frac{k^2(\rho)}{k(0)}\approx 
2\frac{\rho^2}{l^2}+\mathcal{O}(\rho^3).
\end{align}
Therefore, for large~$N$ the best asymptotic behavior 
of \eqref{eq:isobound} 
is achieved with~$\rho(N)=cN^{-\frac{1}{3}}$ 
for the squared exponential kernel under uniform sampling
and leads 
to~$\hat{\sigma}_{\rho}^2(1)\approx \mathcal{O}
(N^{-\frac{2}{3}})$.
The same approach can 
be used to calculate the information radius~$\rho(N)$ with 
the best asymptotic behavior of the bound in
\cref{cor:var_bound} for the Mat\'ern kernel
 with~$\nu=\frac{1}{2}$. This leads
to~$\rho(N)=cN^{-\frac{1}{2}}$ and an asymptotic 
behavior of~$\hat{\sigma}_{\rho}^2(1)\approx 
\mathcal{O}(N^{-\frac{1}{2}})$. For the 
non-isotropic kernels, we pursue a similar 
approach and substitute the expected number of
samples~$N\rho$ in \eqref{eq:sigbound}, 
which results in the asymptotically
optimal~$\rho(N)=cN^{-\frac{1}{2}}$ and 
$\hat{\sigma}_{\rho}^2(1)\approx 
\mathcal{O}(N^{-\frac{1}{2}})$. For these functions
$\rho(N)$, the posterior variance 
bound~$\bar{\sigma}_{\rho}^2(1)$ from \cref{th:var_bound}
and the bound~$\hat{\sigma}_{\rho}^2(1)$ from 
\cref{cor:var_bound} together with the 
exact posterior variance $\sigma_{\mathrm{num}}^2(1)$ 
averaged over~$20$ different 
training data sets are illustrated 
in \cref{fig:varUni}. \looseness=-1

\begin{figure}
	\centering
	\begin{minipage}{0.75\textwidth}
	\centering
		\begin{tikzpicture} 
		\begin{axis}[%
		hide axis,
		xmin=10,
		xmax=50,
		ymin=0,
		ymax=0.4,
		legend columns = 4,
		legend style={draw=white!15!black,legend cell align=left}
		]
		\addlegendimage{black}
		\addlegendentry{$\sigma^2_{\mathrm{num}}(1)$};
		\addlegendimage{blue}
		\addlegendentry{$\hat{\sigma}_{\rho}^2(1)$ from Cor.~\ref{cor:var_bound}};
		\addlegendimage{red}
		\addlegendentry{$\bar{\sigma}_{\rho}^2(1)$ from Thm.~\ref{th:var_bound}};
		\end{axis}
		\end{tikzpicture}
	\end{minipage}\\
	\begin{minipage}{0.47\textwidth}
		\centering
		\def\file{var_uniform_squaredexponential.txt}
		\tikzsetnextfilename{var_SE}
		\begin{tikzpicture}
		\begin{loglogaxis}[xmin=0,xmax=1220,ymin=0.00001,ymax=5.2, samples=100,
		grid=none, axis y line=left, axis x line=bottom, 
		ylabel={$\sigma_N^2$}, 
		scaled x ticks=false,legend pos=south west, 
		]
		\addplot[black] table[x = idx,y = sig_m ]{\file};
		\addplot[blue] table[x = idx,y = sig_bm ]{\file};
		\addplot[red] table[x = idx,y = sig_bm_gen ]{\file};
		\end{loglogaxis}
		\end{tikzpicture}
	\end{minipage}\hfill
	\begin{minipage}{0.47\textwidth}
		\centering
		\def\file{var_uniform_exponential.txt}
		\tikzsetnextfilename{var_exp}
		\begin{tikzpicture}
		\begin{loglogaxis}[xmin=0,xmax=1220,ymin=0.001,ymax=5.2, samples=100,
		grid=none, axis y line=left, axis x line=bottom, 
		ylabel={$\sigma_N^2$}, 
		scaled x ticks=false, 
		legend columns = 2, 
		legend style={at={(-0.215,-0.55)}, anchor= west}]
		\addplot[black] table[x = idx,y = sig_m ]{\file};
		\addplot[blue] table[x = idx,y = sig_bm ]{\file};
		\addplot[red] table[x = idx,y = sig_bm_gen ]{\file};
		\end{loglogaxis}
		\end{tikzpicture}
	\end{minipage}
\begin{minipage}{0.47\textwidth}
		\centering
		\def\file{var_uniform_polynomial.txt}
		\tikzsetnextfilename{var_poly}
		\begin{tikzpicture}
		\begin{loglogaxis}[xmin=0,xmax=1220,ymin=0.00002,ymax=5.2, samples=100,
		grid=none, axis y line=left, axis x line=bottom, 
		xlabel=$N$,ylabel={$\sigma_N^2$}, 
		scaled x ticks=false,legend pos=south west, 
		]
		\addplot[black] table[x = idx,y = sig_m ]{\file};
		\addplot[red] table[x = idx,y = sig_bm_gen ]{\file};
		\end{loglogaxis}
		\end{tikzpicture}
	\end{minipage}\hfill
	\begin{minipage}{0.47\textwidth}
		\centering
		\def\file{var_uniform_neuralnetwork.txt}
		\tikzsetnextfilename{var_nn}
		\begin{tikzpicture}
		\begin{loglogaxis}[xmin=0,xmax=1220,ymin=0.00005,ymax=.5, samples=100,
		grid=none, axis y line=left, axis x line=bottom, 
		xlabel=$N$,ylabel={$\sigma_N^2$}, 
		scaled x ticks=false, 
		legend columns = 2, 
		legend style={at={(-0.215,-0.55)}, anchor= west}]
		\addplot[black] table[x = idx,y = sig_m ]{\file};
		\addplot[red] table[x = idx,y = sig_bm_gen ]{\file};
		\end{loglogaxis}
		\end{tikzpicture}
	\end{minipage}

	\vspace{-0.25cm}	
	\caption{Average posterior variance and bounds 
		of the squared exponential (top left), the Mat\'ern
		kernel with~$\nu=\frac{1}{2}$ (top right), the 
		polynomial kernel with~$p=3$ (bottom left) and the 
		neural network kernel (bottom right) for 
		uniformly sampled training data}
	\label{fig:varUni}
\end{figure}
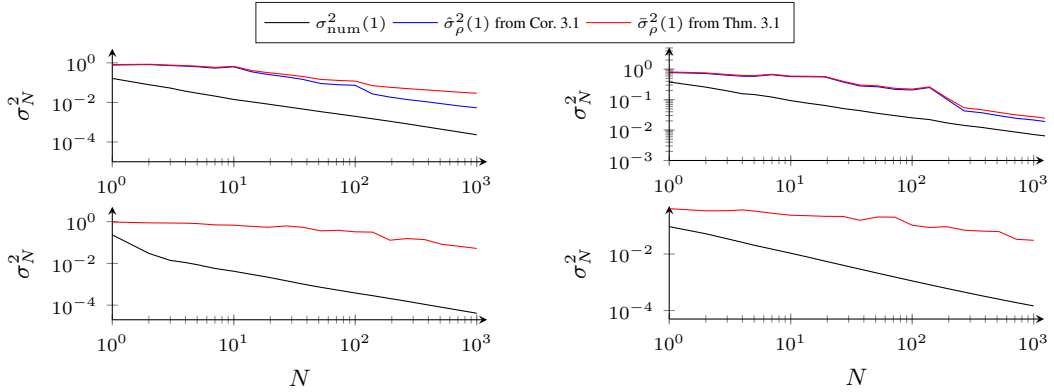

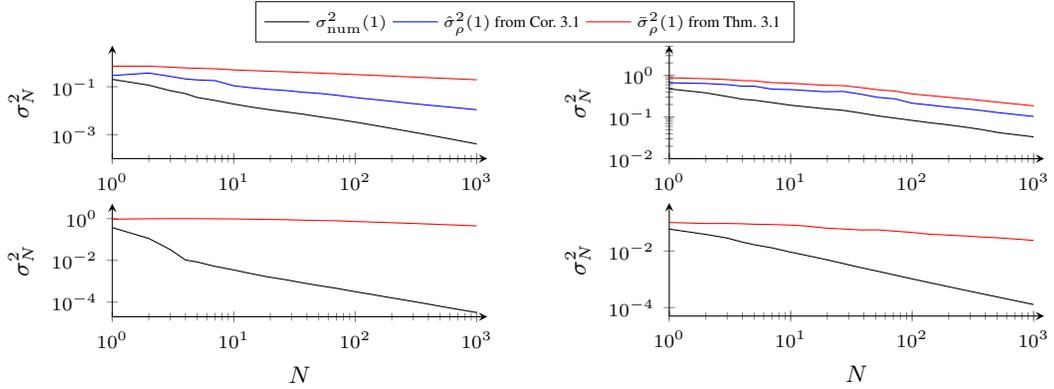
\begin{figure}[t]
	\center
	\begin{minipage}{0.75\textwidth}
	\center
		\begin{tikzpicture} 
		\begin{axis}[%
		hide axis,
		xmin=10,
		xmax=50,
		ymin=0,
		ymax=0.4,
		legend columns = 4,
		legend style={draw=white!15!black,legend cell align=left}
		]
		\addlegendimage{black}
		\addlegendentry{$\sigma^2_{\mathrm{num}}(1)$};
		\addlegendimage{blue}
		\addlegendentry{$\hat{\sigma}_{\rho}^2(1)$ from Cor.~\ref{cor:var_bound}};
		\addlegendimage{red}
		\addlegendentry{$\bar{\sigma}_{\rho}^2(1)$ from Thm.~\ref{th:var_bound}};
		\end{axis}
		\end{tikzpicture}
	\end{minipage}\\
	\begin{minipage}{0.47\textwidth}
	\centering
	\def\file{var_vanishing_squaredexponential.txt}
	\tikzsetnextfilename{var_SE_vanishing}
	\begin{tikzpicture}
	\begin{loglogaxis}[xmin=0,xmax=1220,ymin=0.0001,
	ymax=5.2, samples=100,
	grid=none, axis y line=left, axis x line=bottom, 
	ylabel={$\sigma_N^2$}, 
	scaled x ticks=false,legend pos=south west, 
	]
	\addplot[black] table[x = idx,y = sig_m ]{\file};
	\addplot[blue] table[x = idx,y = sig_bm ]{\file};
	\addplot[red] table[x = idx,y = sig_bm_gen ]{\file};
	\end{loglogaxis}
	\end{tikzpicture}
	\end{minipage}\hfill
	\begin{minipage}{0.47\textwidth}
		\centering
		\def\file{var_vanishing_exponential.txt}
		\tikzsetnextfilename{var_exp_vanishing}
		\begin{tikzpicture}
		\begin{loglogaxis}[xmin=0,xmax=1220,ymin=0.01,
		ymax=5.2, samples=100, grid=none, axis y line=left, 
		axis x line=bottom, ylabel={$\sigma_N^2$},
		scaled x ticks=false, 
		legend columns = 2, 
		legend style={at={(-0.215,-0.55)}, anchor= west}]
		\addplot[black] table[x = idx,y = sig_m ]{\file};
		\addplot[blue] table[x = idx,y = sig_bm ]{\file};
		\addplot[red] table[x = idx,y = sig_bm_gen ]{\file};
		\end{loglogaxis}
		\end{tikzpicture}
	\end{minipage}
\begin{minipage}{0.47\textwidth}
		\centering
		\def\file{var_vanishing_polynomial.txt}
		\tikzsetnextfilename{var_poly}
		\begin{tikzpicture}
		\begin{loglogaxis}[xmin=0,xmax=1220,ymin=0.00002,ymax=5.2, samples=100,
		grid=none, axis y line=left, axis x line=bottom, 
		xlabel=$N$,ylabel={$\sigma_N^2$}, 
		scaled x ticks=false,legend pos=south west, 
		]
		\addplot[black] table[x = idx,y = sig_m ]{\file};
		\addplot[red] table[x = idx,y = sig_bm_gen ]{\file};
		\end{loglogaxis}
		\end{tikzpicture}
	\end{minipage}\hfill
	\begin{minipage}{0.47\textwidth}
		\centering
		\def\file{var_vanishing_neuralnetwork.txt}
		\tikzsetnextfilename{var_nn}
		\begin{tikzpicture}
		\begin{loglogaxis}[xmin=0,xmax=1220,ymin=0.00005,ymax=.5, samples=100,
		grid=none, axis y line=left, axis x line=bottom, 
		xlabel=$N$,ylabel={$\sigma_N^2$}, 
		scaled x ticks=false, 
		legend columns = 2, 
		legend style={at={(-0.215,-0.55)}, anchor= west}]
		\addplot[black] table[x = idx,y = sig_m ]{\file};
		\addplot[red] table[x = idx,y = sig_bm_gen ]{\file};
		\end{loglogaxis}
		\end{tikzpicture}
	\end{minipage}		
	
	\vspace{-0.25cm}
	\caption{Average posterior variance and bounds 
		of the squared exponential (top left), the Mat\'ern
		kernel with~$\nu=\frac{1}{2}$ (top right), the 
		polynomial kernel with~$p=3$ (bottom left) and the 
		neural network kernel (bottom right) for 
		training data sampled from vanishing distribution}
	\label{fig:VarVan}
\end{figure}

We also compare the bounds in \cref{th:var_bound} and 
\cref{cor:var_bound} to the exact 
posterior variance for training data sampled from 
the distribution with density function
\begin{align}
p(x)=4|1-x|,\quad 0.5\leq x\leq 1.5.
\end{align}
This probability density vanishes at the test point 
$x=1$ and it leads to~$\tilde{p}(N)=4\rho^2(N)$
for~\mbox{$\rho(N)\leq 0.5$}. By employing a Taylor 
expansion of the kernel around the test point,
we can derive the optimal asymptotic decay rates
for~$\rho(N)$ as in the previous section. 
For the isotropic and the Mat\'ern kernel, this leads to 
$\rho(N)=cN^{-\frac{1}{3}}$ and an asymptotic 
behavior of the posterior variance~$\sigma_N^2(1)\approx
\mathcal{O}(N^{-\frac{1}{3}})$. For the squared 
exponential kernel, a slightly faster decreasing 
$\rho(N)=cN^{-\frac{1}{4}}$ can be chosen, which 
results in~$\hat{\sigma}_{\rho}^2(1)\approx 
\mathcal{O}(N^{-\frac{1}{2}})$.
The curves for the bounds~$\bar{\sigma}_{\rho}^2(1)$ 
from \cref{th:var_bound} and~$\hat{\sigma}_{\rho}^2(1)$
from \cref{cor:var_bound} as well as the exact posterior 
variance averaged over~$20$ different training 
data sets for 
the vanishing training sample distribution are 
illustrated in \cref{fig:VarVan}. \looseness=-1

The posterior variance bounds for the isotropic
squared exponential and Mat\'ern kernel exhibit 
a similar decrease rate as the actually observed 
one in \cref{fig:varUni} and \cref{fig:VarVan}. 
Indeed, the bound for the
Mat\'ern kernel shows the exact same behavior
and only differs by a constant factor for large~$N$. 
However, for non-isotropic kernels, our 
bound in \cref{th:var_bound} is rather loose 
as it converges with~$\mathcal{O}(N^{-\frac{1}{2}})$ 
while the true posterior variance exhibits a decay
rate of approximately~$\mathcal{O}(N^{-1})$ for the 
uniform distribution in \cref{fig:varUni}. Furthermore, 
no difference of the decrease rate of the numerically
estimated posterior variance can be observed 
between both figures, whereas our bound decreases slightly
slower for the vanishing probability distribution in \cref{fig:VarVan}. 
These two observations are caused
by the non-isotropy of these kernels: they consider
data globally, while our bound only decreases when
training points are added locally around the test point.
However, this problem can be overcome by exploiting
the special structure of these bounds similarly as in 
\cref{cor:var_bound}, e.g., by using a more suitable distance
metric in \cref{th:var_bound} to define the information radius~$\rho$. 
Furthermore, the guaranteed decay 
rate of the variance is already sufficient to ensure that the 
uniform error bounds in~\cite{Srinivas2012,Chowdhury2017a}
converge to zero for kernels such as, e.g., the linear 
covariance kernel.\looseness=-1

\subsection{Average Learning Curves}
\label{subsec:alc}
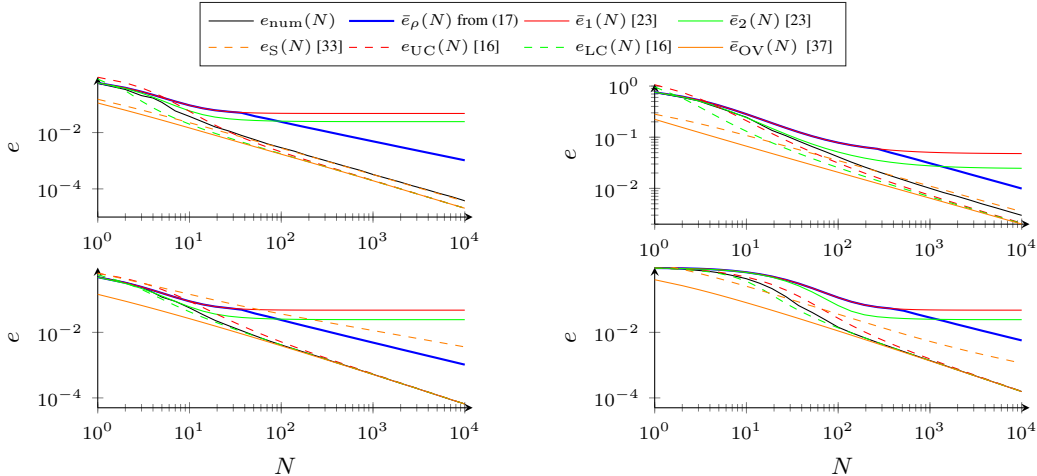
\begin{figure}[!t]
	\pgfplotsset{width=5\columnwidth /5, compat = 1.13, 
		height = 52.5\columnwidth /100, grid= major, 
		legend cell align = left, ticklabel style = {font=\scriptsize},
		every axis label/.append style={font=\small},
		legend style = {font=\tiny},title style={yshift=-7pt, font = \small} }
	
	\center
	\begin{minipage}{0.75\textwidth}
	\center
		\begin{tikzpicture} 
		\begin{axis}[%
		hide axis,
		xmin=10,
		xmax=50,
		ymin=0,
		ymax=0.4,
		legend columns = 4,
		legend style={draw=white!15!black,legend cell align=left}
		]
		\addlegendimage{black}
		\addlegendentry{$e_{\mathrm{num}}(N)$};
		\addlegendimage{blue,thick}
		\addlegendentry{$\bar{e}_{\rho}(N)$ from \eqref{eq:alc_rho}};
		\addlegendimage{red}
		\addlegendentry{$\bar{e}_1(N)$~\cite{Williams2000} };
		\addlegendimage{green}
		\addlegendentry{$\bar{e}_2(N)$~\cite{Williams2000} };
		\addlegendimage{orange,dashed}
		\addlegendentry{$e_{\mathrm{S}}(N)$~\cite{Sarkka2013} };
		\addlegendimage{red,dashed}
		\addlegendentry{$e_{\mathrm{UC}}(N)$~\cite{Sollich2002} };
		\addlegendimage{green,dashed}
		\addlegendentry{$e_{\mathrm{LC}}(N)$~\cite{Sollich2002} };
		\addlegendimage{orange}
		\addlegendentry{$\bar{e}_{\mathrm{OV}}(N)$ ~\cite{Opper1999} };
		\end{axis}
		\end{tikzpicture}
	\end{minipage}\\
	\begin{minipage}{0.47\textwidth}
		\centering
		\def\file{alcsquaredexponential.txt}
		\tikzsetnextfilename{alc_SE}
		\begin{tikzpicture}
		\begin{loglogaxis}[xmin=0,xmax=12200,ymin=0.00001,
		ymax=0.95, samples=100, grid=none, axis y line=left, 
		axis x line=bottom,ylabel={$e$}, 
		scaled x ticks=false,legend pos=south west, 
		]
		\addplot[black] table[x = idx,y = y_exact ]{\file};
		\addplot[blue,thick] table[x = idx,y = y_bound ]{\file};
		\addplot[red] table[x = idx,y = yE1 ]{\file};
		\addplot[green] table[x = idx,y = yE2 ]{\file};
		\addplot[red,dashed] table[x = idx,y = y_uc ]{\file};
		\addplot[green,dashed] table[x = idx,y = y_lc ]{\file};
		\addplot[orange,dashed] table[x = idx,y = y_S ]{\file};
		\addplot[orange] table[x = idx,y = y_ov ]{\file};
		\end{loglogaxis}
		\end{tikzpicture}
	\end{minipage}\hfill
	\begin{minipage}{0.47\textwidth}
		\centering
		\def\file{alcexponential.txt}
		\tikzsetnextfilename{alc_exp}
		\begin{tikzpicture}
		\begin{loglogaxis}[xmin=0,xmax=12200,ymin=0.002,
		ymax=1.1, samples=100, grid=none, axis y line=left, 
		axis x line=bottom, ylabel={$e$}, 
		scaled x ticks=false,legend pos=south west, 
		legend columns = 4, 
		legend style={at={(-0.215,-0.55)}, anchor= west}]
		\addplot[black] table[x = idx,y = y_exact ]{\file};
		\addplot[blue,thick] table[x = idx,y = y_bound ]{\file};
		\addplot[red] table[x = idx,y = yE1 ]{\file};
		\addplot[green] table[x = idx,y = yE2 ]{\file};
		\addplot[red,dashed] table[x = idx,y = y_uc ]{\file};
		\addplot[green,dashed] table[x = idx,y = y_lc ]{\file};
		\addplot[orange,dashed] table[x = idx,y = y_S ]{\file};
		\addplot[orange] table[x = idx,y = y_ov ]{\file};
		\end{loglogaxis}
		\end{tikzpicture}
	\end{minipage}\\
	\begin{minipage}{0.47\textwidth}
		\centering
		\def\file{alcrationalquadratic.txt}
		\tikzsetnextfilename{alc_ratquad}
		\begin{tikzpicture}
		\begin{loglogaxis}[xmin=0,xmax=12200,ymin=0.00005,
		ymax=0.95, samples=100, grid=none, axis y line=left, 
		axis x line=bottom, xlabel=$N$,ylabel={$e$}, 
		scaled x ticks=false,legend pos=south west, 
		]
		\addplot[black] table[x = idx,y = y_exact ]{\file};
		\addplot[blue,thick] table[x = idx,y = y_bound ]{\file};
		\addplot[red] table[x = idx,y = yE1 ]{\file};
		\addplot[green] table[x = idx,y = yE2 ]{\file};
		\addplot[red,dashed] table[x = idx,y = y_uc ]{\file};
		\addplot[green,dashed] table[x = idx,y = y_lc ]{\file};
		\addplot[orange,dashed] table[x = idx,y = y_S ]{\file};
		\addplot[orange] table[x = idx,y = y_ov ]{\file};
		\end{loglogaxis}
		\end{tikzpicture}
	\end{minipage}\hfill
	\begin{minipage}{0.47\textwidth}
		\centering
		\def\file{alcperiodic.txt}
		\tikzsetnextfilename{alc_per}
		\begin{tikzpicture}
		\begin{loglogaxis}[xmin=0,xmax=12200,ymin=0.00005,
		ymax=0.95, samples=100, grid=none, axis y line=left, 
		axis x line=bottom, xlabel=$N$,ylabel={$e$}, 
		scaled x ticks=false,legend pos=south west, 
		legend columns = 4, 
		legend style={at={(-0.215,-0.55)}, anchor= west}]
		\addplot[black] table[x = idx,y = y_exact ]{\file};
		\addplot[blue,thick] table[x = idx,y = y_bound ]{\file};
		\addplot[red] table[x = idx,y = yE1 ]{\file};
		\addplot[green] table[x = idx,y = yE2 ]{\file};
		\addplot[red,dashed] table[x = idx,y = y_uc ]{\file};
		\addplot[green,dashed] table[x = idx,y = y_lc ]{\file};
		\addplot[orange,dashed] table[x = idx,y = y_S ]{\file};
		\addplot[orange] table[x = idx,y = y_ov ]{\file};
		\end{loglogaxis}
		\end{tikzpicture}
	\end{minipage}

	\vspace{-0.25cm}
	\caption{Average learning curve approximations and bounds for the squared exponential (top left),
		the Mat\'ern kernel (top right), the rational quadratic kernel (bottom left) and the 
		periodic kernel (bottom right); the novel bound \eqref{eq:alc_rho} converges to zero in contrast to the existing
		upper bounds from \cite{Williams2000} \looseness=-1}
	\label{fig:avl}
	\vspace{-0.0cm}
\end{figure}

We pursue a greedy approach to choose~$n$ in our learning
curve bound \eqref{eq:alc_rho}. We start with 
$n=1$ at~$N=1$ and increase~$n$ until it reaches a 
local minimum. For~$N>1$, we start with the value of
~$n$ from the previous step and perform the 
same optimization. Note, that the bound \eqref{eq:alc_rho} 
is only defined for~$n>1$. Therefore, we make use 
of \eqref{eq:alc_1} for~$n=1$. We compare our learning
curve bound \eqref{eq:alc_rho} to a numerical approximation 
of the learning curve~$e_{\mathrm{num}}(N)$ obtained 
by averaging over~$1000$ test points and~$50$ training 
data sets for each point in the average learning curve. 
Furthermore, we evaluate the lower and upper continuous 
average learning curve approximations~$e_{\mathrm{LC}}(N)$
and~$e_{\mathrm{UC}}(N)$~\cite{Sollich2002}, respectively, 
as well as the approximation suggested in~\cite{Sarkka2013},
which are based on spectral methods. Moreover, we compare 
our bound to the average learning curve bounds 
\eqref{eq:alc_1} and \eqref{eq:alc2}
proposed in~\cite{Williams2000}. Finally, 
the lower bound derived in~\cite{Opper1999} is evaluated.
The results of this comparison for the squared exponential, the Mat\'ern, the rational quadratic and the periodic kernel
with~$l=0.3$ and noise variance~$\sigma_n^2=0.05$ are depicted in \cref{fig:avl}. 
Note that~$\sigma_n^2$ has been 
subtracted from all curves for illustrative purposes.\looseness=-1

Due to the use of \eqref{eq:alc_1} in our average learning curve bound 
for~$n=1$, both curves are identical 
at the beginning of the plots in \cref{fig:avl}. 
However, for large
$N$ our bound outperforms both average learning curve
bounds~$\bar{e}_1(N)$ and~$\bar{e}_2(N)$.
In comparison to the average learning curve 
approximations~$e_S(N)$,~$e_{\mathrm{UC}}(N)$
and~$e_{\mathrm{UL}}(N)$ our average learning curve
bound typically differs more strongly from the numerical
learning curve~$e_{\mathrm{num}}(N)$ as 
depicted in \cref{fig:avl}.
However, these are only approximations, hence
there is no guarantee that they do not intersect
with the true average learning curve. In fact,
intersections with~$e_{\mathrm{num}}(N)$ can
be observed for most of the kernels in 
\cref{fig:avl}. Moreover, it should be
noted that the asymptotic behavior
of our bound usually does not differ a lot
from the true average learning curve. 
In
fact, we can observe the true decay rate of 
$\mathcal{O}(N^{-\frac{1}{2}})$ for the 
Mat\'ern kernel~\cite{Opper1997}.\looseness=-1

\widowpenalty=10000
\section{Conclusion}
\label{sec:conclusion}
In this paper we present a novel bound for the 
posterior variance of Gaussian processes with 
Lipschitz continuous kernels. We 
develop conditions that guarantee its convergence 
to zero and investigate probability distributions 
that satisfy these conditions. Furthermore, we 
demonstrate how the bound can be specialized 
to smaller classes of kernels and extend it to 
average learning curve bounds, which can be used 
for a learning comparison between different kernels.

\medskip

\small
\bibliographystyle{IEEEtran}
\bibliography{example_paper}

\begin{thebibliography}{10}
\providecommand{\url}[1]{#1}
\csname url@samestyle\endcsname
\providecommand{\newblock}{\relax}
\providecommand{\bibinfo}[2]{#2}
\providecommand{\BIBentrySTDinterwordspacing}{\spaceskip=0pt\relax}
\providecommand{\BIBentryALTinterwordstretchfactor}{4}
\providecommand{\BIBentryALTinterwordspacing}{\spaceskip=\fontdimen2\font plus
\BIBentryALTinterwordstretchfactor\fontdimen3\font minus
  \fontdimen4\font\relax}
\providecommand{\BIBforeignlanguage}[2]{{%
\expandafter\ifx\csname l@#1\endcsname\relax
\typeout{** WARNING: IEEEtran.bst: No hyphenation pattern has been}%
\typeout{** loaded for the language `#1'. Using the pattern for}%
\typeout{** the default language instead.}%
\else
\language=\csname l@#1\endcsname
\fi
#2}}
\providecommand{\BIBdecl}{\relax}
\BIBdecl

\bibitem{Rasmussen2006}
C.~E. Rasmussen and C.~K.~I. Williams, \emph{{Gaussian processes for machine
  learning}}.\hskip 1em plus 0.5em minus 0.4em\relax The MIT Press, 2006.

\bibitem{Berkenkamp}
F.~Berkenkamp, A.~Krause, and A.~P. Schoellig, ``{Bayesian Optimization with
  Safety Constraints: Safe Automatic Parameter Tuning in Robotics},'' ETH
  Z{\"{u}}rich, Z{\"{u}}rich, Tech. Rep., 2016.

\bibitem{Berkenkamp2016}
F.~Berkenkamp, A.~P. Schoellig, and A.~Krause, ``{Safe Controller Optimization
  for Quadrotors with Gaussian Processes},'' in \emph{Proceedings of the IEEE
  International Conference on Robotics and Automation}, 2016, pp. 491--496.

\bibitem{Berkenkamp2017}
F.~Berkenkamp, M.~Turchetta, A.~P. Schoellig, and A.~Krause, ``{Safe
  Model-based Reinforcement Learning with Stability Guarantees},'' in
  \emph{Advances in Neural Information Processing Systems}, 2017, pp. 908--918.

\bibitem{Koller2018}
T.~Koller, F.~Berkenkamp, M.~Turchetta, and A.~Krause, ``{Learning-based Model
  Predictive Control for Safe Exploration and Reinforcement Learning},'' in
  \emph{Proceedings of the IEEE Conference on Decision and Control}, 2018.

\bibitem{Berkenkamp2015}
F.~Berkenkamp and A.~P. Schoellig, ``{Safe and Robust Learning Control with
  Gaussian Processes},'' in \emph{Proceedings of the European Control
  Conference}, 2015, pp. 2496--2501.

\bibitem{Umlauft2017}
J.~Umlauft, T.~Beckers, M.~Kimmel, and S.~Hirche, ``{Feedback Linearization
  using Gaussian Processes},'' in \emph{Proceedings of the IEEE Conference on
  Decision and Control}, 2017, pp. 5249--5255.

\bibitem{Beckers2018}
T.~Beckers and S.~Hirche, ``{Gaussian Process based Passivation of a Class of
  Nonlinear Systems with Unknown Dynamics},'' in \emph{Proceedings of the
  European Control Conference}, 2018.

\bibitem{Umlauft2018a}
J.~Umlauft, T.~Beckers, and S.~Hirche, ``{Scenario-based Optimal Control for
  Gaussian Process State Space Models},'' in \emph{Proceedings of the European
  Control Conference}, 2018.

\bibitem{Umlauft2018}
J.~Umlauft, L.~P{\"{o}}hler, and S.~Hirche, ``{An Uncertainty-Based Control
  Lyapunov Approach for Control-Affine Systems Modeled by Gaussian Process},''
  \emph{IEEE Control Systems Letters}, vol.~2, no.~3, pp. 483--488, 2018.

\bibitem{Helwa2018}
M.~K. Helwa, A.~Heins, and A.~P. Schoellig, ``{Provably Robust Learning-Based
  Approach for High-Accuracy Tracking Control of Lagrangian Systems},'' in
  \emph{Proceedings of the IEEE Conference on Decision and Control}, 2018.

\bibitem{Srinivas2012}
N.~Srinivas, A.~Krause, S.~M. Kakade, and M.~W. Seeger,
  ``{Information-theoretic regret bounds for Gaussian process optimization in
  the bandit setting},'' \emph{IEEE Transactions on Information Theory},
  vol.~58, no.~5, pp. 3250--3265, 2012.

\bibitem{Chowdhury2017a}
S.~R. Chowdhury and A.~Gopalan, ``{On Kernelized Multi-armed Bandits},'' in
  \emph{Proceedings of the International Conference on Machine Learning}, 2017,
  pp. 844--853.

\bibitem{Sollich1999}
P.~Sollich, ``{Learning Curves for Gaussian Processes},'' in \emph{Advances in
  Neural Information Processing Systems}, 1999, pp. 344--350.

\bibitem{Malzahn2001}
D.~Malzahn and M.~Opper, ``{Learning Curves for Gaussian Processes Regression:
  A Framework for Good Approximations},'' \emph{Advances in Neural Information
  Processing Systems 13}, pp. 273--279, 2001.

\bibitem{Sollich2002}
P.~Sollich and A.~Halees, ``{Learning Curves for Gaussian Process Regression:
  Approximations and Bounds},'' \emph{Neural Computation}, vol.~14, pp.
  1393--1428, 2002.

\bibitem{LeGratiet2014}
L.~{Le Gratiet} and J.~Garnier, ``{Asymptotic Analysis of the Learning Curve
  for Gaussian Process Regression},'' \emph{Machine Learning}, vol.~98, no.~3,
  pp. 407--433, 2014.

\bibitem{Xu2011}
Y.~Xu, J.~Choi, and S.~Oh, ``{Mobile Sensor Network Navigation using Gaussian
  Processes with Truncated Observations},'' \emph{IEEE Transactions on
  Robotics}, vol.~27, no.~6, pp. 1118--1131, 2011.

\bibitem{Schulz2015}
E.~Schulz, J.~B. Tenenbaum, D.~N. Reshef, M.~Speekenbrink, and S.~J. Gershman,
  ``{Assessing the Perceived Predictability of Functions},'' in
  \emph{Proceedings of the Conference of the Cognitive Science Society}, 2015,
  pp. 2116--2121.

\bibitem{Ueno2018}
T.~Ueno, H.~Hino, A.~Hashimoto, Y.~Takeichi, M.~Sawada, and K.~Ono, ``{Adaptive
  Design of an X-ray Magnetic Circular Dichroism Spectroscopy Experiment with
  Gaussian Process Modeling},'' \emph{npj Computational Materials}, vol.~4,
  no.~1, pp. 1--8, 2018.

\bibitem{Reeb2018}
D.~Reeb, A.~Doerr, S.~Gerwinn, and B.~Rakitsch, ``{Learning Gaussian Processes
  by Minimizing PAC-Bayesian Generalization Bounds},'' in \emph{Advances in
  Neural Information Processing Systems}, 2018.

\bibitem{Krause2008}
A.~Krause, A.~Singh, and C.~Guestrin, ``{Near-optimal Sensor Placements in
  Gaussian Processes: Theory, Efficient Algorithms and Empirical Studies},''
  \emph{Journal of Machine Learning Research}, vol.~9, pp. 235--284, 2008.

\bibitem{Williams2000}
C.~K.~I. Williams and F.~Vivarelli, ``{Upper and Lower Bounds on the Learning
  Curve for Gaussian Processes},'' \emph{Machine Learning}, vol.~40, pp.
  77--102, 2000.

\bibitem{Vivarelli1998}
F.~Vivarelli, ``{Studies on the Generalisation of Gaussian Processes and
  Bayesian Neural Networks},'' Ph.D. dissertation, Aston University, 1998.

\bibitem{Shekhar2018}
S.~Shekhar and T.~Javidi, ``{Gaussian Process Bandits with Adaptive
  Discretization},'' \emph{Electronic Journal of Statistics}, vol.~12, pp.
  3829--3874, 2018.

\bibitem{Stein1999}
M.~L. Stein, \emph{{Interpolation of Spatial Data: Some Theory for
  Kriging}}.\hskip 1em plus 0.5em minus 0.4em\relax Springer Science {\&}
  Business Media, 1999.

\bibitem{Kanagawa2018}
\BIBentryALTinterwordspacing
M.~Kanagawa, P.~Hennig, D.~Sejdinovic, and B.~K. Sriperumbudur, ``{Gaussian
  Processes and Kernel Methods: A Review on Connections and Equivalences},''
  pp. 1--64, 2018. [Online]. Available: \url{http://arxiv.org/abs/1807.02582}
\BIBentrySTDinterwordspacing

\bibitem{Wu1993}
Z.~M. Wu and R.~Schaback, ``{Local Error Estimates for Radial Basis Function
  Interpolation of Scattered Data},'' \emph{IMA Journal of Numerical Analysis},
  vol.~13, no.~1, pp. 13--27, 1993.

\bibitem{Wendland2005}
H.~Wendland, \emph{{Scattered Data Approximation}}.\hskip 1em plus 0.5em minus
  0.4em\relax Cambridge University Press, 2004.

\bibitem{Schaback2006}
R.~Schaback and H.~Wendland, ``{Kernel Techniques : From Machine Learning to
  Meshless Methods},'' \emph{Acta Numerica}, vol.~15, pp. 543--639, 2006.

\bibitem{Beatson2010}
R.~Beatson, O.~Davydov, and J.~Levesley, ``{Error Bounds for Anisotropic RBF
  Interpolation},'' \emph{Journal of Approximation Theory}, vol. 162, no.~3,
  pp. 512--527, 2010.

\bibitem{Scheuerer2013}
M.~Scheuerer, R.~Schaback, and M.~Schlather, ``{Interpolation of Spatial Data -
  A Stochastic or a Deterministic Problem ?}'' \emph{European Journal of
  Applied Mathematics}, vol.~24, no.~4, pp. 601--629, 2013.

\bibitem{Sarkka2013}
S.~S{\"{a}}rkk{\"{a}} and A.~Solin, ``{Continuous-space Gaussian Process
  Regression and Generalized Wiener Filtering with Application to Learning
  Curves},'' in \emph{Image Analysis}, J.-K. K{\"{a}}m{\"{a}}r{\"{a}}inen and
  M.~Koskela, Eds.\hskip 1em plus 0.5em minus 0.4em\relax Springer Berlin
  Heidelberg, 2013, pp. 172--181.

\bibitem{Urry2013}
M.~J. Urry and P.~Sollich, ``{Random Walk Kernels and Learning Curves for
  Gaussian Process Regression on Random Graphs},'' \emph{Journal of Machine
  Learning Research}, vol.~14, pp. 1801--1835, 2013.

\bibitem{Chai2009}
K.~M. Chai, ``{Generalization Errors and Learning Curves for Regression with
  Multi-task Gaussian Processes},'' \emph{Advances in Neural Information
  Processing Systems}, pp. 1--9, 2009.

\bibitem{Ashton2012}
S.~R.~F. Ashton and P.~Sollich, ``{Learning Curves for Multi-task Gaussian
  Process Regression},'' in \emph{Advances in Neural Information Processing
  Systems}, 2012, pp. 1393--1428.

\bibitem{Opper1999}
M.~Opper and F.~Vivarelli, ``{General Bounds on Bayes Errors for Regression
  with Gaussian Processes},'' \emph{Advances in Neural Information Processing
  Systems}, pp. 302--308, 1999.

\bibitem{Opper1997}
M.~Opper, ``{Regression with Gaussian Processes: Average Case Performance},''
  in \emph{Hong Kong International Workshop on Theoretical Aspects of Neural
  Computation: A Multidisciplinary Perspective}.\hskip 1em plus 0.5em minus
  0.4em\relax World Scientific, 1997, pp. 17--23.

\bibitem{Gershgorin1931}
S.~Gershgorin, ``{Ueber die Abgrenzung der Eigenwerte einer Matrix},''
  \emph{Bulletin de l'Academie des Sciences de l'URSS. Classe des sciences
  mathematiques et na}, no.~6, pp. 749--754, 1931.

\bibitem{Forbes2011}
C.~Forbes, M.~Evans, N.~Hastings, and B.~Peacock, \emph{{Statistical
  Distributions}}, 4th~ed.\hskip 1em plus 0.5em minus 0.4em\relax Hoboken, New
  Jersey: Wiley, 2011.

\bibitem{Cormen2009}
T.~H. Cormen, C.~E. Leiserson, R.~L. Rivest, and C.~Stein, \emph{{Introduction
  to Algorithms}}, 3rd~ed.\hskip 1em plus 0.5em minus 0.4em\relax Cambridge,
  Massachusetts: The MIT Press, 2009.

\end{thebibliography}

\newpage
\appendix

\section{Posterior Variance Bound and Asymptotic Behavior}
\begin{proof}[Proof of Theorem 3.1]
	Since $\bm{K}_N+\sigma_n^2\bm{I}_N$ is a 
	positive definite, quadratic matrix, it follows 
	that
	\begin{align*}
	\sigma_{N}^2(\bm{x})&\leq k(\bm{x},\bm{x})-
	\frac{\left\|\bm{k}_N(\bm{x})\right\|^2}
	{\lambda_{\max}\left(\bm{K}_N\right)+\sigma_n^2}.
	\end{align*}
	Applying the Gershgorin theorem \cite{Gershgorin1931} 
	the maximal eigenvalue is bounded by
	\begin{align*}
	\lambda_{\max}(\bm{K}_N)\leq N
	\max\limits_{\bm{x}',\bm{x}''\in\mathbb{D}_N^x}
	k(\bm{x}',\bm{x}'').
	\end{align*}
	Furthermore, due to the definition of 
	$\bm{k}_N(\bm{x})$ we have
	\begin{align*}
	\|\bm{k}_N(\bm{x})\|^2\geq N 
	\min\limits_{\bm{x}'\in\mathbb{D}_N^x}
	k^2(\bm{x}',\bm{x}).
	\end{align*}
	Therefore, $\sigma_N^2(\bm{x})$ can be bounded by
	\begin{align}
	\sigma_{N}^2(\bm{x})&\leq k(\bm{x},\bm{x})-
	\frac{N\min\limits_{\bm{x}'\in\mathbb{D}_N^x}
		k^2(\bm{x}',\bm{x})}{N
		\max\limits_{\bm{x}',\bm{x}''\in\mathbb{D}_N^x}
		k(\bm{x}',\bm{x}'')+\sigma_n^2}.
	\label{eq:sigma_bound1}
	\end{align}
	This bound can be further simplified exploiting 
	the fact that $\sigma_N^2(\bm{x})\leq
	\sigma_{N-1}^2(\bm{x})$ \cite{Vivarelli1998} 
	and considering only samples inside the ball 
	$\mathbb{B}_{\rho}(\bm{x})$ with radius 
	$\rho\in\mathbb{R}_+$. Using this reduced data 
	set instead of $\mathbb{D}_N^x$ and writing the 
	right side of \eqref{eq:sigma_bound1} as a single 
	fraction results in
	\begin{align}
	\label{eq:sigma_bound3}
	\sigma_{N}^2(\bm{x})&\leq\frac{k(\bm{x},\bm{x})
		\sigma_n^2+\left|\mathbb{B}_{\rho}(\bm{x})\right|
		\xi(\bm{x},\rho)}{\left|\mathbb{B}_{\rho}
		(\bm{x})\right|\max\limits_{\bm{x}',\bm{x}''
			\in\mathbb{B}_{\rho}(\bm{x})}k(\bm{x}',\bm{x}'')
		+\sigma_n^2},
	\end{align}
	where
	\begin{align*}
	\xi(\bm{x},\rho&)=
	k(\bm{x},\bm{x})\max\limits_{\bm{x}',\bm{x}''
		\in\mathbb{B}_{\rho}(\bm{x})}k(\bm{x}',\bm{x}'')
	-\min\limits_{\bm{x}'\in \mathbb{B}_{\rho}
		(\bm{x})}k^2(\bm{x}',\bm{x}).
	\end{align*}
	Under the assumption that $\rho\leq 
	\frac{k(\bm{x},\bm{x})}{L_k}$ it follows from 
	the Lipschitz continuity of $k(\cdot,\cdot)$ that
	\begin{align*}
	\min\limits_{\bm{x}'\in \mathbb{B}_{\rho}(\bm{x})}
	k^2(\bm{x}',\bm{x})\geq (k(\bm{x},\bm{x})-
	L_k\rho)^2.
	\end{align*}
	Furthermore, it holds that
	\begin{align*}
	\max\limits_{\bm{x}',\bm{x}''\in\mathbb{B}_{\rho}
		(\bm{x})}k(\bm{x}',\bm{x}'')\leq 
	k(\bm{x},\bm{x})+2L_k\rho.
	\end{align*}
	Therefore, $\xi(\bm{x},\rho)$ can be bounded by
	\begin{align*}
	\xi(\bm{x},\rho)&\leq 4k(\bm{x},\bm{x})L_k\rho
	-L_k^2\rho^2.
	\end{align*}
	Hence, the result is proven.
\end{proof}

\begin{proof}[Proof of Corollary 3.1]
	The proof follows directly from 
	\eqref{eq:sigma_bound3} and the fact that 
	\begin{align*}
	\min\limits_{\bm{x}'\in\mathbb{B}_{\rho}(\bm{x})}
	k(\bm{x}',\bm{x})&\leq k(\rho)\\
	\max\limits_{\bm{x}',\bm{x}''\in\mathbb{B}_{\rho}
		(\bm{x})}k(\bm{x}',\bm{x}'')&= k(0)
	\end{align*}
	since the kernel is isotropic and decreasing.
\end{proof}

\begin{proof}[Proof of Corollary 3.2]
	The upper bound in Theorem 3.1 converges 
	to zero due to the assumptions on $\rho(N)$ and 
	$\left|\mathbb{B}_{\rho(N)}(\bm{x})\right|$. Hence, 
	convergence of $\sigma_N^2(\bm{x})$ to zero is 
	implied.
\end{proof}

\section{Conditions on Probability Distributions for Asymptotic Convergence}

In order to prove Theorem 3.2, some auxiliary results
for binomial distributions are necessary. These
are provided in the following Lemmas.
\begin{lemma}
	\label{lem:Bernoulli}
	The $k$-th central moment of a Bernoulli distributed random variable $X$ is given by
	\begin{align}
	E[(X-E[X])^k]=\sum\limits_{i=0}^{k-1}(-1)^i\binom{k}{i}p^{i+1}+p^k
	\end{align}
\end{lemma}
\begin{proof}
	The polynom $(X-E[X])^k$ can be expanded as
	\begin{align*}
	(X-&E[X])^k=\sum\limits_{i=0}^k \binom{k}{i}(-1)^{i}X^{k-i}E[X]^i.
	\end{align*}
	The $k$-th moment about the origin of the Bernoulli distribution is given by $p$ for $k>0$ \cite{Forbes2011}. Therefore, the expectation of this polynomial is given by
	\begin{align*}
	E[(X-&E[X])^k]=\sum\limits_{i=0}^{k-1} \binom{k}{i}(-1)^{i}pp^i+p^k,
	\end{align*}
	which directly yields the result.
\end{proof}

\begin{lemma}
	\label{lem2}
	The $2k$-th central moment of a binomial distributed random variable $M$ with $N>2k$ samples is bounded by
	\begin{align}
	E[(X-E[X])^{2k}]\leq \sum\limits_{m=1}^k(Np)^m\alpha_m
	\label{eq3}
	\end{align}
	where $\alpha_m\in\mathbb{R}$ are finite coefficients.
\end{lemma}
\begin{proof}
	A binomial random variable is defined as the sum of $N$ i.i.d. Bernoulli random variables $X_i$. Therefore, the $2k$-th central moment of the binomial distribution is given by
	\begin{align}
	E[(M-E[M])^{2k}]=E\left[\left(\sum\limits_{i=1}^{N}(X_i-p)\right)^{2k}\right].
	\end{align}
	Define the multinomial coefficient as
	\begin{align}
	\binom{N}{i_1,i_2,\ldots,i_k}=\frac{N!}{\prod\limits_{j=1}^ki_j!}.
	\end{align}
	Then, the sum in the expectation can be expanded, which yields
	\begin{align}
	\label{eq2}
	E[(M-E[M])^{2k}]=
	\sum\limits_{i_1+i_2+\ldots+i_{N}=2k}\binom{2k}{i_1,i_2,\ldots,i_{N}}\prod\limits_{j=1}^{N}E\left[\left(X_j-p)\right)^{i_j}\right].
	\end{align}
	This equation expresses the moments of the binomial distribution in terms of the moments of the Bernoulli distribution. Since the first central moment of every distribution equals $0$, summands containing a $i_j=1$ equal $0$. Therefore, we obtain the equality
	\begin{align}
	\label{eq1}
	E[(M-E[M])^{2k}]=
	\sum\limits_{\subalign{i_1+i_2+\ldots+i_{N}=2k\\i_j\neq 1\forall j=1,\ldots,N}}\binom{2k}{i_1,i_2,\ldots,i_{N}}\prod\limits_{i_j>1}E\left[\left(X_j-p)\right)^{i_j}\right].
	\end{align}
	Moreover, we have
	\begin{align}
	E[(X-E[X])^k]=p h_k(p)
	\end{align}
	with
	\begin{align}
	h_k(p)=\sum\limits_{i=0}^{k-1}(-1)^i\binom{k}{i}p^{i}+p^{k-1}
	\end{align}
	due to \cref{lem:Bernoulli}. By substituting this into \eqref{eq1} we obtain  
	\begin{align}
	E[(M-E[M])^{2k}]=
	\sum\limits_{\subalign{i_1+i_2+\ldots+i_{N}=2k\\i_j\neq 1\forall j=1,\ldots,N}}\binom{2k}{i_1,i_2,\ldots,i_N}\prod\limits_{i_j>1}ph_{i_j}(p).
	\end{align}
	The product can have between $1$ and $k$ factors due to the structure of the problem. Therefore, it is not necessary for the sum to consider all $N$ coefficients $i_j$, but rather consider only $1\leq m\leq k$ coefficients which are greater than $1$. This leads to the following equality
	\begin{align}
	E[(M-E[M])^{2k}]= \sum\limits_{m=1}^k\binom{N}{m}p^m\sum\limits_{\subalign{i_1+i_2+\ldots+i_{m}=2k\\i_j> 1\forall j=1,\ldots,m}}\binom{2k}{i_1,i_2,\ldots,i_m}\prod\limits_{i_j>1}h_{i_j}(p).
	\label{eq5}
	\end{align}
	Due to \cite{Cormen2009} it holds that $\binom{N}{m}\leq \frac{N^m}{m!}$. Furthermore, the functions $h_k(\cdot)$ can be upper bounded by $\sum\limits_{i=1}^k\binom{k}{i}=2^k$ because $0\leq p\leq 1$. Therefore, we can upper bound the $2k$-th central moment of the binomial distribution by	
	\begin{align}
	E[(M-E[M])^{2k}]\leq \sum\limits_{m=1}^k(Np)^m\alpha_m
	\end{align}
	with
	\begin{align}
	\label{eq4}
	\alpha_m&=\frac{\sum\limits_{\subalign{i_1+i_2+\ldots+i_{m}=2k\\i_j> 1\forall j=1,\ldots,m}}\binom{2k}{i_1,i_2,\ldots,i_m}\prod\limits_{i_j>1}2^{i_j}}{m!}
	\end{align}
	and the result is proven.
\end{proof}
The restriction to $N>2k$ samples allows to derive a relatively simple expression for the expansion in \eqref{eq2}. However, the bound \eqref{eq3} also holds without this condition, since it only guarantees that for $i_j=1$, $\forall j=1,\ldots,N$, $\sum\limits_{j=1}^N i_j\geq 2k$ and therefore, all possible combinations of $i_j$ can be estimated simpler in \eqref{eq5}. Hence, the corresponding summands in \eqref{eq4} can be considered $0$ for $N\leq 2k$ and the upper bound \eqref{eq3} still holds for $N\leq 2k$.

\begin{proof}[Proof of Theorem 3.2]
	We have to show 
	that the number of samples from the probability 
	distribution with density $p(\cdot)$ inside the 
	balls with radius $\rho(N)$ grows to infinity for 
	$N\rightarrow\infty$. The number of samples 
	$|\mathbb{B}_{\rho(N)}(\bm{x})|$ follows a 
	binomial distribution with mean
	\begin{align*}
	E\left[\left|\mathbb{B}_{\rho(N)}(\bm{x})
	\right|\right]&=N\tilde{p}(N),
	\end{align*}
	where
	\begin{align*}
	\tilde{p}(N)=\int\limits_{\{\bm{x}'\in\mathbb{X}:
		\|\bm{x}-\bm{x}'\|\leq \rho(N)\}}p(\bm{x}')
	\mathrm{d}\bm{x}'
	\end{align*}
	is the probability of a sample lying inside 
	the ball around $\bm{x}$ with radius $\rho(N)$ 
	for fixed $N\in\mathbb{N}$. Since we have 
	\begin{align}
	\int\limits_{\{\bm{x}'\in\mathbb{X}:\|\bm{x}
		-\bm{x}'\|\leq \rho(N)\}}p(\bm{x}')\mathrm{d}\bm{x}'
	&\geq cN^{-1+\epsilon}
	\label{eq:cond2new}
	\end{align}
	by assumption, this mean goes to infinity, 
	i.e.
	\begin{align*}
	\lim\limits_{N\rightarrow\infty}E\left[
	\left|\mathbb{B}_{\rho(N)}(\bm{x})\right|\right]
	=\lim\limits_{N\rightarrow\infty}cN^{\epsilon}
	=\infty.
	\end{align*}
	Therefore, it is sufficient to show that 
	$|\mathbb{B}_{\rho(N)}(\bm{x})|$ converges to 
	its expectation almost surely, which is 
	identically to proving that 
	\begin{align*}
	\lim\limits_{N\rightarrow\infty}\frac{
		|\mathbb{B}_{\rho(N)}(\bm{x})|}{E\left[
		\left|\mathbb{B}_{\rho(N)}(\bm{x})\right|
		\right]}=1\quad a.s.
	\end{align*}
	Due to the Borel-Cantelli lemma, this convergence 
	is guaranteed if 
	\begin{align}
	\sum\limits_{N=1}^{\infty}P\left( \left|\frac{
		|\mathbb{B}_{\rho(N)}(\bm{x})|}{E\left[\left|
		\mathbb{B}_{\rho(N)}(\bm{x})\right|\right]}-1\right|
	>\xi \right)<\infty
	\label{eq:bclemma}
	\end{align}
	holds for all $\xi>0$. The probability for each 
	$N\in\mathbb{N}$ can be bounded by 
	\begin{align*}
	P\Bigg( \Bigg| \frac{|\mathbb{B}_{\rho(N)}(\bm{x})|}
	{N\tilde{p}(N)}-1\Bigg|&>\xi \Bigg)\leq 
	\frac{	E\left[\left(\left|\mathbb{B}_{\rho(N)}
		(\bm{x})\right|-N\tilde{p}(N)\right)^{2k}\right]}
	{(\xi N\tilde{p}(N))^{2k}}.
	\end{align*}
	for each $k\in\mathbb{N}_+$ due to Chebyshev's
	inequality, where the $2k$-th central moment of 
	the binomial distribution can be bounded by
	\begin{align*}
	E\left[\left(\left|\mathbb{B}_{\rho(N)}(\bm{x})
	\right|-N\tilde{p}(N)\right)^{2k}\right]\leq
	\sum\limits_{i=1}^{k}\alpha_i\tilde{p}^i(N)
	N^i
	\end{align*} 
	with some coefficients $\alpha_i<\infty$ due to
	\cref{lem2}.
	Therefore, we can bound each probability in 
	\eqref{eq:bclemma} by
	\begin{align*}
	P\Bigg( \Bigg| \frac{|\mathbb{B}_{\rho(N)}
		(\bm{x})|}{N\tilde{p}(N)}-1\Bigg|&>\xi \Bigg)
	\leq \sum\limits_{i=1}^{k}\alpha_i
	\tilde{p}^{-2k+i}(N)N^{-2k+i}.
	\end{align*}
	Due to \eqref{eq:cond2new} this bound can be 
	simplified to 
	\begin{align*}
	P\Bigg( \Bigg| \frac{|\mathbb{B}_{\rho(N)}
		(\bm{x})|}{N\tilde{p}(N)}-1\Bigg|&>\xi \Bigg)
	\leq N^{-k\epsilon}\sum\limits_{i=0}^{k-1}
	\tilde{\alpha}_{k-i}N^{-i\epsilon},
	\end{align*}
	where $\tilde{\alpha}_i=c^{-2k+i}\alpha_i$.
	Let $k=\left\lceil\frac{1}{\epsilon}\right\rceil+1$. 
	Then, each exponent is smaller than or equal to 
	$-1-\epsilon$. Hence, the sum of probabilities can 
	be bounded by 
	\begin{align*}
	\sum\limits_{N=1}^{\infty}P\Bigg( \Bigg| 
	\frac{|\mathbb{B}_{\rho(N)}(\bm{x})|}
	{N\tilde{p}(N)}-1\Bigg|&>\xi \Bigg)\leq 
	\sum\limits_{i=0}^{k-1}\tilde{\alpha}_{k-i} 
	\zeta\big((k+i)\epsilon\big),
	\end{align*}
	where $\zeta(\cdot)$ is the Riemann zeta function, 
	which has finite values. Therefore, we obtain
	\begin{align*}
	\sum\limits_{N=1}^{\infty}P\Bigg( \Bigg|
	\frac{|\mathbb{B}_{\rho(N)}(\bm{x})|}
	{N\tilde{p}(N)}-1\Bigg|>\epsilon \Bigg)< 
	\infty
	\end{align*}
	and consequently, the theorem is proven.
\end{proof}

\begin{proof}[Proof of Corollary 3.3]
	Let 
	\begin{align*}
	\bar{p}&=\min\limits_{\|\bm{x}-\bm{x}'\|\leq 
		\xi}p(\bm{x}')\\
	\tilde{p}(N)&=\int\limits_{\{\bm{x}'\bm{x}'
		\in\mathbb{X}:\|\bm{x}-\bm{x}'\|\leq 
		\xi\}}p(\bm{x}')\mathrm{d}\bm{x}',
	\end{align*}
	where $\bar{p}$ is positive by assumption. Then, 
	we can bound $\tilde{p}(N)$ by
	\begin{align*}
	\tilde{p}(N)\geq \bar{p}V_{d}\rho^{d}(N),
	\end{align*}
	where $V_{d}$ is the volume of the $d$ 
	dimensional unit ball. Since $\rho(N)\geq c
	N^{-\frac{1}{d}+\epsilon}$ for some 
	$c,\epsilon>0$ by assumption, it follows that 
	\begin{align*}
	\tilde{p}(N)\geq \bar{p}V_{d}cN^{-1+
		\frac{\epsilon}{d}}.
	\end{align*}
	Hence, $\tilde{p}(N)$ satisfies the conditions 
	of Theorem 3.2, which proves the corollary.
\end{proof}

\end{document}